\def\isarxiv{1} %%% for icml submission version, we comment this line

\ifdefined\isarxiv
\documentclass[11pt]{article}

\usepackage[numbers]{natbib}

\else
\documentclass[10pt,twocolumn,letterpaper]{article}

\fi

\usepackage{amsmath}
\usepackage{amsthm}
\usepackage{amssymb}
\usepackage{algorithm}
\usepackage{subfig}
\usepackage{algpseudocode}
\usepackage{graphicx}
\usepackage{grffile}
\usepackage{wrapfig,epsfig}
\ifdefined\isarxiv
\usepackage{url}
\fi
\usepackage{xcolor}
\usepackage{epstopdf}
\usepackage{bbm}
\usepackage{dsfont}

\ifdefined\isarxiv
\else

\usepackage[review]{iccv}      % To produce the REVIEW version
\input{preamble}
\fi

\allowdisplaybreaks

\ifdefined\isarxiv

\usepackage{tikz}
\usepackage{hyperref}  %%% arxiv don't allow this.
\hypersetup{colorlinks=true,citecolor=blue,linkcolor=blue} %%% Zhao : maybe we should comment this in submission.
\usetikzlibrary{arrows}
\usepackage[margin=1in]{geometry}

\else

\definecolor{iccvblue}{rgb}{0.21,0.49,0.74}
\usepackage[pagebackref,breaklinks,colorlinks,allcolors=iccvblue]{hyperref}

\fi

\graphicspath{{./figs/}}

\theoremstyle{plain}
\newtheorem{theorem}{Theorem}[section]
\newtheorem{lemma}[theorem]{Lemma}
\newtheorem{definition}[theorem]{Definition}

\newcommand{\wh}{\widehat}

\newcommand{\N}{\mathcal{N}}
\newcommand{\R}{\mathbb{R}}

\renewcommand{\d}{\mathrm{d}}

\newcommand{\X}{\mathsf{X}}
\newcommand{\Y}{\mathsf{Y}}
\newcommand{\Z}{\mathsf{Z}}

\newcommand{\F}{\mathsf{F}}
\newcommand{\V}{\mathsf{V}}

\DeclareMathOperator{\diag}{diag}

\makeatletter
\newcommand*{\RN}[1]{\expandafter\@slowromancap\romannumeral #1@}
\makeatother

\ifdefined\isarxiv
\else
\title{HOFAR: High-Order Augmentation of Flow Autoregressive Transformers}

\author{First Author\\
Institution1\\
Institution1 address\\
{\tt\small firstauthor@i1.org}
\and
Second Author\\
Institution2\\
First line of institution2 address\\
{\tt\small secondauthor@i2.org}
}
\fi

\begin{document}

\ifdefined\isarxiv

\date{}
\title{HOFAR: High-Order Augmentation of Flow Autoregressive Transformers}
\author{
Yingyu Liang\thanks{\texttt{
yingyul@hku.hk}. The University of Hong Kong. \texttt{
yliang@cs.wisc.edu}. University of Wisconsin-Madison.} 
\and
Zhizhou Sha\thanks{\texttt{
shazz20@mails.tsinghua.edu.cn}. Tsinghua University.}
\and
Zhenmei Shi\thanks{\texttt{
zhmeishi@cs.wisc.edu}. University of Wisconsin-Madison.}
\and
Zhao Song\thanks{\texttt{ magic.linuxkde@gmail.com}. The Simons Institute for the Theory of Computing at UC Berkeley.}
\and
Mingda Wan\thanks{\texttt{
dylan.r.mathison@gmail.com}. Anhui University.}
}

\fi

\ifdefined\isarxiv
\begin{titlepage}
  \maketitle
  \begin{abstract}
    Flow Matching and Transformer architectures have demonstrated remarkable performance in image generation tasks, with recent work FlowAR [Ren et al., 2024] synergistically integrating both paradigms to advance synthesis fidelity. However, current FlowAR implementations remain constrained by first-order trajectory modeling during the generation process. This paper introduces a novel framework that systematically enhances flow autoregressive transformers through high-order supervision. We provide theoretical analysis and empirical evaluation showing that our High-Order FlowAR (HOFAR) demonstrates measurable improvements in generation quality compared to baseline models. The proposed approach advances the understanding of flow-based autoregressive modeling by introducing a systematic framework for analyzing trajectory dynamics through high-order expansion.

  \end{abstract}
  \thispagestyle{empty}
\end{titlepage}

{\hypersetup{linkcolor=black}
\tableofcontents
}
\newpage

\else

\maketitle
\begin{abstract}
    Flow Matching and Transformer architectures have demonstrated remarkable performance in image generation tasks, with recent work FlowAR [Ren et al., 2024] synergistically integrating both paradigms to advance synthesis fidelity. However, current FlowAR implementations remain constrained by first-order trajectory modeling during the generation process. This paper introduces a novel framework that systematically enhances flow autoregressive transformers through high-order supervision. We provide theoretical analysis and empirical evaluation showing that our High-Order FlowAR (HOFAR) demonstrates measurable improvements in generation quality compared to baseline models. The proposed approach advances the understanding of flow-based autoregressive modeling by introducing a systematic framework for analyzing trajectory dynamics through high-order expansion.

\end{abstract}

\fi

\section{Introduction}

Recently, flow-matching \cite{lcb+22} and diffusion models \cite{hja20} have demonstrated remarkable capabilities in the field of image generation \cite{rbl+22,ekb+24}. Several works have explored extending these models to generate images with an additional dimension, such as incorporating a temporal dimension for video generation \cite{sph+22,lcw+23} or a 3D spatial dimension for 3D object generation \cite{xxmp24,m24}. Even 4D generation \cite{zcw+25,lyx+24} has become feasible using diffusion models.
Another prominent line of research focuses on auto-regressive models, where the Transformer framework has achieved groundbreaking success in natural language processing. Models such as GPT-4 \cite{gpt4}, Gemini 2 \cite{gemini2}, and DeepSeek \cite{gyz+25} have significantly impacted millions of users worldwide. 

Given the success of the auto-regressive generation paradigm and the Transformer framework, recent works have explored integrating auto-regressive generation into image generation. A representative example is the Visual Auto-Regressive (VAR) model \cite{tjy+25}, which introduces hierarchical image generation with different image patches. Other works, such as FlowAR \cite{ryh+24} and ARFlow \cite{hzy+24}, integrate flow-matching with auto-regressive generation.
However, these existing approaches primarily focus on modeling the direct transition path between the prior distribution and the target image distribution, paying less attention to high-order dynamics. High-order dynamics play a crucial role in capturing complex dependencies between different modalities, which is especially important for tasks like video generation that require long-term coherence. Moreover, high-order supervision enhances a model’s generalization ability by encouraging it to learn fundamental generative principles rather than relying on lower-order patterns.

Motivated by these insights, we propose High-Order FlowAR \textbf{(HOFAR)}, an approach that builds upon the strengths of auto-regressive models and flow-matching techniques while extending them to model higher-order interactions. By explicitly incorporating high-order dynamics, HOFAR improves realism, coherence, and generalization in generative tasks. We theoretically prove that HOFAR maintains computational efficiency compared to its base models while empirically demonstrating its superior performance.

In summary, our contributions are as follows:
\begin{itemize}
    \item We introduce HOFAR, a novel framework that integrates high-order dynamics into flow-matching-based auto-regressive generation, enhancing the model’s ability to capture complex dependencies.
    \item We provide a theoretical analysis showing that HOFAR maintains computational efficiency while benefiting from high-order modeling.
    \item We conduct empirical evaluations demonstrating that HOFAR achieves improved generation quality, coherence, and generalization compared to existing auto-regressive generative models.
\end{itemize}

 %%% Section 1. Introduction
\section{Related Works}

\subsection{Flow-based and Diffusion-based Generative Models}
Flow-based and diffusion-based generative models have demonstrated significant potential in image and video generation tasks \cite{hja20,hhs23,ltl+24,gxx+24,jwt+24, xzl+24, lzw+24}. Among these, Latent Diffusion Models (LDM) \cite{rbl+22} have emerged as a particularly powerful approach, especially in the domain of text-to-image synthesis. Recent advancements, such as Stable Diffusion V3 \cite{ekb+24}, have integrated flow-matching techniques as an alternative strategy to further improve generation quality and enhance the photorealism of synthesized images. Moreover, a growing body of research \cite{jsl+24,wsd+24,wcz+23,wxz+24} has highlighted the potential of combining the strengths of diffusion models and flow-matching models to achieve even greater generation fidelity. In this context, we acknowledge several influential works in flow-matching and diffusion-based generation \cite{hst+22,swyy23,dy24,hsz+24,wfq+24, ccl+25, lss+25_unreveling, ssz+24_dit, llss24_softmax, hwl+24,cll+25,kls+25,cgl+25_rich,kll+25,lsss24,lss+24_multi_layer, gkl+25}, which have greatly inspired our research. 

\subsection{High-Order Dynamic Supervision}
High-order dynamics are often overlooked in the research community, despite their critical role in modeling target distributions—such as image or video distributions—with greater accuracy and effectiveness. Current research primarily explores high-order dynamics within gradient-based methods. For example, solvers~\cite{dng+22,hlj+24} and regularization frameworks~\cite{kbjd20,fjno20} for neural ordinary differential equations (neural ODEs)~\cite{crbd18,gcb+18} frequently leverage higher-order derivatives to enhance performance~\cite{rck+24, cgl+25, lss+25}. Beyond machine learning, the study of higher-order temporal Taylor methods (TTMs) has been extensively applied to solving both stiff~\cite{cc94} and non-stiff~\cite{cc94,cc82} systems, demonstrating their broad utility in computational mathematics.

\paragraph{Roadmap.} 
This paper is organized as follows:
Section~\ref{sec:preliminary} introduces the fundamental notations used throughout the paper and provides formal definitions for each module in the proposed model.
In Section~\ref{sec:main_results}, we present the training and inference algorithms for our HOFAR model, along with an analysis of its computational efficiency.
In Section~\ref{sec:technical_overview}, we delve into the technical details and methodologies employed to prove our formal theorem.
In Section~\ref{sec:experiments}, we conduct an empirical evaluation of the HOFAR model, showcasing its effectiveness and robustness in image generation tasks.
Finally, in Section~\ref{sec:conclusion}, we summarize the key contributions of this paper and provide concluding remarks.

\section{Preliminary} \label{sec:preliminary}

In this section, we introduce the formal mathematical definitions for the FlowAR model and our High-Order FlowAR (HOFAR) model. These definitions provide the foundational framework for understanding the preprocessing, downsampling, upsampling, and transformer-based components of the proposed architecture.
In Section~\ref{sec:pre:notations}, we introduce the notations we used in this work. 
In Section~\ref{sec:pre:flowar_preprocessing_process}, we describe the preprocessing steps applied to input images before they are fed into the model. 
In Section~\ref{sec:pre:autoregresive_transformer_architecture}, we detail the autoregressive Transformer architecture, which generates the conditional embeddings utilized by the flow-matching components in the FlowAR model. 
Finally, in Section~\ref{sec:pre:flow_matching_architecture}, we provide a formal mathematical definition of flow modeling and present the implementation of the flow-matching architecture.

\subsection{Notations}\label{sec:pre:notations}
Given a matrix $X \in \R^{hw \times d}$, we denote its tensorized form as $\X \in \R^{h \times w \times d}$. Additionally, we define the set $[n]$ to represent $\{1,2,\cdots, n\}$ for any positive integer $n$. We define the set of natural numbers as $\mathbb{N}:= \{0,1,2,\dots\}$. Let $X \in \mathbb{R}^{m \times n}$ be a matrix, where $X_{i,j}$ refers to the element at the $i$-th row and $j$-th column. When $x_i$ belongs to $\{ 0,1 \}^*$, it signifies a binary number with arbitrary length. In a general setting, $x_i$ represents a length $p$ binary string, with each bit taking a value of either 1 or 0. Given a matrix $X \in \R^{n \times d}$, we define $\|X\|_\infty  $ as the maximum norm of $X$. Specifically, $\|X\|_\infty = \max_{i,j} |X_{i,j}|$. 

\subsection{FlowAR Preprocessing Process} \label{sec:pre:flowar_preprocessing_process}

We begin by introducing the preprocessing procedure of the FlowAR model. The image is first passed through a Variational Autoencoder (VAE) to obtain a latent image embedding before being processed by the main body of the FlowAR model. 

Let $\X \in \R^{h \times w \times c}$ denote the image embedding generated by the VAE, where $h$, $w$, and $c$ represent the height, width, and number of channels, respectively. The next step involves downsampling the image embedding $\X$ to multiple scales. To formalize this process, we first define the linear downsampling function.

\begin{definition}[Linear Downsampling Function]\label{def:linear_down_sample_function}
If the following conditions hold:
\begin{itemize}
    \item Let $\X \in \R^{h \times w \times c}$ denote the input tensor, where $h,w,c$ represent height, width, and the number of channels, respectively.
    \item Let the positive integer $r \geq 1$ denote the scaling factor.
\end{itemize}
The linear downsampling function $\phi_{\mathrm{down}}(\X,r)$ computes an output tensor $\Y \in \R^{(h/r) \times (w/r) \times c}$. 

To be more specific, let $\Phi_{\mathrm{down}} \in \R^{(h/r \cdot w/r) \times hw}$ denote a linear transformation matrix. 
The downsampling transformation consists of three steps:
\begin{itemize}
    \item Reshape $\X$ into the matrix $X \in \R^{hw \times c}$ by flattening its spatial dimensions.  
    \item Apply the linear transformation matrix $\Phi_{\mathrm{down}}$ on $\X$ as 
    \begin{align*}
         Y = \Phi_{\mathrm{down}}X \in \R^{(h/r \cdot w/r) \times c},
    \end{align*}
    \item Reshaped back to $\Y \in \R^{(h/r) \times(w/r) \times c}$.
\end{itemize}
\end{definition}

Next, we define the multi-scale downsampling tokenizer, which leverages the linear downsampling function to generate a sequence of token maps at multiple scales.

\begin{definition}[Multi-Scale Downsampling Tokenizer]\label{def:downsample_tokenizer}
If the following conditions hold:
\begin{itemize}
 \item Let $\X\in \R^{h \times w \times c}$ denote the image embedding generated by VAE.
 \item Let $K \in \mathbb{N}$ denote the number of scales.
 \item Let the positive integer $a \geq 1$ denote the base scaling factor. 
 \item For $i \in [K]$, we define scale-specific factors $r_i := a^{K-i}$ and use the linear downsampling function $\phi_{\mathrm{down}}(\X,r_i)$ from Definition~\ref{def:linear_down_sample_function}.
\end{itemize}
We define the multi-scale downsampling tokenizer as $\mathsf{TN}(\X) := \{\Y^{1}, \dots,\Y^{K}\}$, which outputs a sequence of token maps $\{\Y^2, \Y^2,\dots, \Y^K\}$, where the $i$-th token map is generated by
 \begin{align*}
     \Y^i := \phi_{\mathrm{down},i}(\X,r_i) \in \R^{(h / r_i) \times (w/r_i) \times c},
 \end{align*}

\end{definition}

During inference, we need to upsample the embeddings after each processing step. To formalize this operation, we define the bicubic upsampling function as follows.

\begin{definition}[Upsampling Function]\label{def:bicubic_up_sample_function}
If the following conditions hold:
\begin{itemize}
    \item Let $\X \in \R^{h \times w \times c}$ denote the input tensor, where $h,w,c$ represent height, width, and the number of channels, respectively.
    \item Let the A positive integer $r \geq 1$ denote the scaling factor.
    \item Let $W:\R \to [0,1]$ denote the bicubic kernel. 
\end{itemize}
We define the bicubic upsampling function as $\phi_{\mathrm{up}}(\X,r)$, which computes $\Y \in \R^{rh \times rw \times c}$. For every output position $i \in [rh], j \in [rw], l \in [c]$:
\begin{align*}
    \Y_{i,j,l} =  \sum_{s=-1}^2 \sum_{t=-1}^2 W(s) \cdot W(t) \cdot \X_{\lfloor \frac{i}{r}\rfloor+s, \lfloor \frac{j}{r}\rfloor+t,l}
\end{align*}

\end{definition}

\subsection{Autoregressive Transformer Architecture} \label{sec:pre:autoregresive_transformer_architecture}

The downsampled embeddings are then fed into the transformer architecture to generate the condition tensor for the flow matching model. The autoregressive transformer is a key component of the FlowAR model. Below, we define its attention layer, feedforward layer, and the overall autoregressive transformer.

\begin{definition}[Attention Layer]\label{def:attn_layer}
If the following conditions hold:
\begin{itemize}
    \item Let $\X \in \R^{h \times w \times c}$ denote the input tensor, where $h,w,c$ represent height, width, and the number of channels, respectively.
    \item Let $W_Q,W_K,W_V \in \R^{c \times c}$ denote the weight matrices, which will be used in query, key, and value projection, respectively.
\end{itemize}
The attention layer $\mathsf{Attn}(\X)$ is defined by computing the output tensor $\Y \in \R^{h \times w \times c}$ in the following three steps:
\begin{itemize}
    \item Reshape $\X$ into a matrix $X \in \R^{hw \times c}$ with spatial dimensions collapsed.
    \item Attention matrix computation. For $i,j \in [hw]$, compute pairwise scores:
    \begin{align*}
        A_{i,j} := & ~\exp(  X_{i,*}   W_Q   W_K^\top   X_{j,*}^\top), \text{~~for~} i, j \in [hw].
    \end{align*}
    \item Normalization. Compute diagnal matrix $D:=\diag(A {\bf 1}_n) \in \R^{hw \times hw}$, where ${\bf 1}_n$ is the all-ones vector. And compute:
    \begin{align*}
         Y := D^{-1}AXW_V \in \R^{hw \times c}.
    \end{align*}
    \item Reshape $Y$ to $\Y \in \R^{h \times w \times c}$.
\end{itemize}
\end{definition}

The feedforward layer is another critical component of the transformer architecture. We define it as follows.

\begin{definition}[Feed Forward Layer]\label{def:ffn}
If the following conditions hold:
\begin{itemize}
    \item Let $\X \in \R^{h \times w \times c}$ denote the input tensor, where $h,w,c$ represent height, width, and the number of channels, respectively.
    \item Let $W_1, W_2 \in \R^{c \times d}$ denote the weight matrices and $b_1, b_2 \in \R^{1 \times d}$ denote the bias vectors.
    \item Let $\sigma:\R \to \R$ denote the $\mathsf{ReLU}$ activation function which is applied element-wise.
\end{itemize}
We defined the feedforward operation as $\Y := \mathsf{FFN}(\X)$.

To be more specific, it computes an output tensor $\Y \in \R^{h \times w \times d}$ in the following steps:
\begin{itemize}
    \item Reshape $\X$ into a matrix $X \in \R^{hw \times c}$ with spatial dimensions collapsed.
    \item For each $j \in [hw]$, compute 
    \begin{align*}
        Y_{j,*}=  \underbrace{X_{j,*}}_{1 \times c} +  \sigma (\underbrace{X_{j,*}}_{1\times c} \cdot \underbrace{W_1}_{c \times c} + \underbrace{b_1}_{1\times c}) \cdot \underbrace{W_2}_{c \times c} + \underbrace{b_2}_{1 \times c} \in \R^{1 \times c}
    \end{align*}
    where $\sigma$ acts element-wise on intermediate results. Then reshape $Y \in \R^{hw \times c}$ into $\Y \in \R^{h \times w \times c}$.
\end{itemize}

\end{definition}

Using the attention and feedforward layers, we now define the autoregressive transformer.

\begin{definition}[Autoregressive Transformer]\label{def:ar_transformer}
If the following conditions hold:
\begin{itemize}
    \item Let $\X \in \R^{h \times w \times c}$ denote the input tensor, where $h,w,c$ represent height, width, and the number of channels, respectively.
    \item Let $K \in \mathbb{N}$ denote the scale number, which is the number of total scales in FlowAR.
    \item For $i \in [K]$, let $\Y_i \in \R^{(h/r_i) \times (w/r_i) \times c}$ denote the token maps generated by the Multi-Scale downsampling tokenizer defined in Definition~\ref{def:downsample_tokenizer} where $r_i = a^{K-i}$ with base $a \in \mathbb{N}^+$.
    \item For $i \in [K]$, let $\phi_{\mathrm{up},i}(\cdot,a): \R^{(h/r_i) \times (w/r_i) \times c}\to \R^{(h/r_{i+1}) \times (w/r_{i+1}) \times c}$ denote the upsampling functions as defined in Definition~\ref{def:bicubic_up_sample_function}.
    \item For $i \in [K]$, let $\mathsf{Attn}_i(\cdot):\R^{(\sum_{j=1}^i h/r_j \cdot w/r_{j})\times c} \to \R^{(\sum_{j=1}^i h/r_j \cdot w/r_{j})\times c}$ denote the attention layer which acts on flattened sequences of dimension defined in Definition~\ref{def:attn_layer}.
    \item For $i \in [K]$, let $\mathsf{FFN}_i(\cdot): \R^{(\sum_{j=1}^i h/r_j \cdot w/r_{j})\times c} \to \R^{(\sum_{j=1}^i h/r_j \cdot w/r_{j})\times c}$ denote the feed forward layer which acts on flattened sequences of dimension defined in Definition~\ref{def:ffn}.
    \item Let $\Z_{\mathrm{init}} \in \R^{(h/r_1) \times (w/r_1) \times c}$ denote the initial condition embedding which encodes class information.
\end{itemize}
Then, the autoregressive processing is:
\begin{itemize}
    \item {\bf Initialization: } Let $\Z_1:=\Z_{\mathrm{init}}$.
    \item {\bf Iterative sequence construction:} For $i \geq 2$.
    \begin{align*}
        \Z_i := \mathsf{Concat}(\mathsf{Z}_{\mathrm{init}}, \phi_{\mathrm{up}, 1}(\Y^1, a), \ldots, \phi_{\mathrm{up}, i-1}(\Y^{i-1}, a)) 
    \end{align*}
    where $\mathsf{Concat}$ reshapes tokens into a unified spatial grid.
    \item {\bf Transformer block:} For $i \in [K]$,
    \begin{align*}
        \mathsf{TF}_i(\Z_i) := \mathsf{FFN_i}(\mathsf{Attn}_i(\Z_i))
        \in \R^{(\sum_{j=1}^i h/r_j \cdot w/r_{j})\times c}
    \end{align*}
    \item {\bf Output decomposition:} Extract the last scale's dimension   from the reshaped $\mathsf{TF}_i(\Z_i)$ to generate $\wh{\Y}_i \in \R^{(h/r_i) \times (w/r_i) \times c}$.
\end{itemize}
\end{definition}

\subsection{Flow Matching Architecture} \label{sec:pre:flow_matching_architecture}

We begin by outlining the concept of velocity flow in the flow-matching architecture. This section introduces the foundational definitions and components necessary to understand the flow-matching model.

\begin{definition}[Flow]\label{def:flow}
If the following conditions hold:
\begin{itemize}
    \item Let $\X \in \R^{h \times w \times c}$ denote the input tensor, where $h,w,c$ represent height, width, and the number of channels, respectively.
    \item Let $K \in \mathbb{N}$ denote the scales number.
    \item For $i \in [K]$, let $\F_i^0 \in \R^{(h / r_i) \times (w/r_i) \times c}$ denote the noise tensor with every entry sampled from $\mathcal{N}(0,1)$.
    \item For $i \in [K]$, let $\wh{\Y}_i \in \R^{(h / r_i) \times (w/r_i) \times c}$ denote the token maps generated by autoregressive transformer as defined in Definition~\ref{def:ar_transformer}.
\end{itemize}
Then, we define the flow model supports the following two operations:
\begin{itemize}
    \item {\bf Interpolation:} For timestep $t \in [0,1]$ and scale $i$,
    \begin{align*}
        \F_i^t := t \wh{\Y}_i + (1-t) \F_i^0
    \end{align*}
    which describes a linear trajectory between the noise $\F_0^i$ and target tokens  $\wh{\Y}_i$.
    \item {\bf Velocity Field:} The time derivative of the flow at scale $i$ is given by
    \begin{align*}
        \V^t_i := \frac{\d \F^t_{i}}{\d t} = \wh{\Y_i} -\F^0_i.
    \end{align*}
    This velocity field is constant across $t$ due to the linear nature of the interpolation.
\end{itemize}
\end{definition}

Before introducing the implementation of the flow-matching model, we first define two essential components: the Multi-Layer Perceptron (MLP) layer and the Layer Normalization (LN) layer. These components are critical for constructing the flow-matching architecture.

\begin{definition}[MLP Layer]\label{def:mlp}
If the following conditions hold:
\begin{itemize}
    \item Let $\X \in \R^{h \times w \times c}$ denote the input tensor, where $h,w,c$ represent height, width, and the number of channels, respectively.
    \item Let $W \in \R^{c \times d}$ denote the weight matrix and $b \in \R^{1 \times d}$ denote the bias vector.
\end{itemize}
We define the MLP layer as $\Y := \mathsf{MLP}(\X,c,d)$, which outputs tensor $\Y \in \R^{h \times w \times d}$ by using the following operations:
\begin{itemize}
    \item Reshape $\X$ into a matrix $X \in \R^{hw \times c}$ with spatial dimensions collapsed.
    \item For all $j \in [hw]$, we apply affine transformation on each row as follows
    \begin{align*}
        Y_{j,*} = \underbrace{X_{j,*}}_{1\times c} \cdot \underbrace{W}_{c \times d} + \underbrace{b}_{1 \times d}
    \end{align*}
    \item Reshape $Y \in \R^{hw \times d}$ into $\Y \in \R^{h \times w \times d}$.
\end{itemize}
\end{definition}

Next, we define the Layer Normalization layer, which is a key component for stabilizing and normalizing the inputs to the flow-matching architecture.
\begin{definition}[Layer Normalization Layer]\label{def:ln}
If the following conditions hold:
\begin{itemize}
    \item Let $\X \in \R^{h \times w \times c}$ denote the input tensor, where $h,w,c$ represent height, width, and the number of channels, respectively.
\end{itemize}
We define the layer normalization as $\Y := \mathsf{LN}(\X)$, which computes $\Y$ through the following steps
\begin{itemize}
    \item Reshape $\X$ into a matrix $X \in \R^{hw \times c}$ with spatial dimensions collapsed.
    \item For each $j \in [hw]$, we apply normalization on each row of the matrix,
    \begin{align*}
        Y_{j,*} = (X_{j,*}-\mu_j) \sigma_j^{-1}
    \end{align*}
    where
    \begin{align*}
        \mu_j := \sum_{k=1}^c X_{j,k}/c, ~~ \sigma_{j} = (\sum_{k=1}^c(X_{j,k}-\mu_j)^2/c)^{1/2}
    \end{align*}
    \item Reshape $Y \in \R^{hw \times c}$ into $\Y \in \R^{h \times w \times c}$.
\end{itemize}
\end{definition}

With the MLP and Layer Normalization layers defined, we now introduce the flow-matching layer, which is a core component of the FlowAR model.

\begin{definition}[Flow Matching Architecture]\label{def:flow_matching_architecture}
If the following conditions hold:
\begin{itemize}
    \item Let $\X \in \R^{h \times w \times c}$ denote the input tensor, where $h,w,c$ represent height, width, and the number of channels, respectively.
    \item Let $K \in \mathbb{N}$ denote the number of total scales in FlowAR.
    \item For $i \in [K]$, let $\wh{\Y}_i \in \R^{(h / r_i) \times (w/r_i) \times c}$ denote the token maps generated by autoregressive transformer defined in Definition~\ref{def:ar_transformer}.
    \item For $i \in [K]$, let $\F_i^t \in \R^{(h / r_i) \times (w/r_i) \times c}$ denote interpolated input defined in Definition~\ref{def:flow}.
    \item For $i \in [K]$, let $t_i \in [0,1]$ denote timestep.
    \item For $i \in [K]$, let $\mathsf{Attn}_i(\cdot):\R^{h/r_i \times w/r_i \times c} \to \R^{h/r_i \times w/r_i \times c}$ denote the attention layer as defined in Definition~\ref{def:attn_layer}.
    \item For $i \in [K]$, let $\mathsf{MLP}_i(\cdot,c,d):\R^{h/r_i \times w/r_i \times c} \to \R^{h/r_i \times w/r_i \times c}$ denote the MLP layer as defined in Definition~\ref{def:mlp}.
    \item For $i \in [K]$, let $\mathsf{LN}_i(\cdot):\R^{h/r_i \times w/r_i \times c} \to \R^{h/r_i \times w/r_i \times c}$ denote the layer norm layer as defined in Definition~\ref{def:ln}.
\end{itemize}
Then we define the flow-matching architecture as $\F''^{t_i}_i := \mathsf{FM}_i(\wh{\Y_i},\F_i^{t_i},t_i)$, which contains the following computation steps:
\begin{itemize}
    \item Generate parameter conditioned on the timestep,
    \begin{align*}
        \alpha_1, \alpha_2, \beta_1, \beta_2, \gamma_1, \gamma_2:=  \mathsf{MLP}_i(\wh{\Y}_i + t_i, c, 6c)
    \end{align*}
    \item Apply attention mechanism,
    \begin{align*}
        \F'^{t_i}_i:= \mathsf{Attn}_i (\gamma_1 \circ \mathsf{LN}(\F_i^{t_i}) + \beta_1) \circ \alpha_1
    \end{align*}
    with $\circ$ denoting Hadamard (element-wise) product.
    \item Apply MLP and LN modules,
    \begin{align*}
        \F''^{t_i}_i := \mathsf{MLP}_i(\gamma_2 \circ \mathsf{LN}(\F'^{t_i}_i)+ \beta_2,c,c) \circ \alpha_2
    \end{align*}
\end{itemize}
\end{definition}

\section{Main Results} \label{sec:main_results}

In this section, we present our theoretical analysis of the computational efficiency of the HOFAR model. We demonstrate that despite incorporating high-order dynamics supervision, the increase in computational complexity for both training and inference remains marginal compared to the significant performance improvements achieved.

\begin{theorem} [Computational Efficiency of HOFAR] \label{thm:hofar_computational_efficiency}
In accordance with Definition~\ref{def:ar_transformer}, the auto-regressive Transformer architecture incorporates $m$ attention layers. The image input $x_{\mathrm{img}} \in \mathbb{R}^{n \times n \times c}$ is encoded with $n^2$ spatial units, $c$ channels, and a $d$-dimensional latent representation. The HOFAR model demonstrates computational costs of $O(kmn^4d^2)$ for both training and inference under the specified structural constraints.  
\end{theorem}

\begin{proof}
The proof follows from Lemma~\ref{lem:hofar_training_running_time} and Lemma \ref{lem:hofar_inference_running_time}. 
\end{proof}

\begin{algorithm}[!ht]\caption{High-Order FlowAR Training} \label{alg:hofar_training}
\begin{algorithmic}[1]
\Procedure{HOFARTraining}{$\theta, D$}
\State {\color{blue} /* $\theta$ denotes the model parameters of $\mathsf{TF}, \mathsf{FM}_{\mathrm{first}}, \mathsf{FM}_{\mathrm{second}}$ */}
\State {\color{blue} /* $D$ denotes the training dataset. */}
\While{not converged}
\State {\color{blue} /* Sample an image from dataset. */}
\State $x_{\mathrm{img}} \sim D$
\State {\color{blue} /* Init loss as $0$. */}
\State $\ell \gets 0$
\State {\color{blue} /* Train the model on $K$ pyramid layers. */}
\For{$i=1 \to K$}
\State {\color{blue} /* Sample random noise. */} \label{line:pyramid_loop_start}
\State $\F^0 \sim \N (0, I)$
\State {\color{blue} /* Sample a random timestep. */}
\State $t \sim [0, 1]$
\State {\color{blue} /* Calculate noisy input. */} \label{line:preparation_start}
\State $\F_{\mathrm{noisy}}^t \gets \alpha_t x_{\mathrm{img}} + \beta_t \F_i^0$ 
\State {\color{blue} /* Calculate first-order ground-truth. */}
\State $\F_{\mathrm{first}}^t \gets \alpha_t' x_{\mathrm{img}} + \beta_t' \F_i^0$
\State {\color{blue} /* Calculate second-order ground-truth. */}
\State $\F_{\mathrm{second}}^t \gets \alpha_t'' x_{\mathrm{img}} + \beta_t'' \F_i^0$ \label{line:preparation_end}
\State {\color{blue} /* Generate condition with Transformer. */}
\State $\wh{\Y} \gets \mathsf{TF}(x_\mathrm{img})$ \label{line:generate_condition_with_transformer}
\State {\color{blue} /* Predict first-order with FM. */} \label{line:flow_matching_pred_start}
\State $\wh{\F}_{\mathrm{first}}^t \gets \mathsf{FM}_{\mathrm{first}}(\F_{\mathrm{noisy}}^t, \wh{\Y})$
\State {\color{blue} /* Predict second-order with FM. */}
\State $\wh{\F}_{\mathrm{second}}^t \gets \mathsf{FM}_{\mathrm{second}}(\F_{\mathrm{noisy}}^t, \wh{\Y})$ \label{line:flow_matching_pred_end}
\State {\color{blue} /* Caculate loss. */}
\State $\ell_c \gets \| \wh{\F}_{\mathrm{first}}^t - \F_{\mathrm{first}}^t \|_2^2 + \| \wh{\F}_{\mathrm{second}}^t - \F_{\mathrm{second}}^t \|_2^2$ \label{line:loss_calculation}
\State $\ell \gets \ell + \ell_c$
\State {\color{blue} /* Downsample $x_{\mathrm{img}}$ for next iteration. */}
\State $x_{\mathrm{img}} \gets \Phi_{\mathrm{down}} x_{\mathrm{img}}$ \label{line:pyramid_loop_end}
\EndFor
\State {\color{blue} /* Optimize parameter $\theta$ with $\ell$. */}
\State $\theta \gets \nabla_\theta ~ \ell$
\EndWhile
\State \Return{$\theta$}
\EndProcedure
\end{algorithmic}
\end{algorithm}

\begin{algorithm}[!ht]\caption{High-Order FlowAR Inference} \label{alg:hofar_inference}
\begin{algorithmic}[1]
\Procedure{HOFARInference}{$c_{\mathrm{input}}$}
\State {\color{blue} /* $c_{\mathrm{input}}$ denotes the condition embedding used for generation. */}
\State {\color{blue} /* Init the Transformer input $x$ with $c_{\mathrm{input}}$. */}
\State $x \gets c_{\mathrm{input}}$
\State {\color{blue} /* Init the $x_{\mathrm{img}}$ with random noise. */}
\State $x_{\mathrm{img}} \gets \N (0, I)$
\State {\color{blue} /* Inference through $K$ pyramid scales. */}
\For{$i=1 \to K$}
\State {\color{blue} /* Pass through the Transformers $\mathsf{TF}$. */} \label{line:inference_loop_start}
\State $\wh{\Y} \gets \mathsf{TF}(x)$ \label{line:inference_transformer_pass}
\State {\color{blue} /* Extract last $i*i$ tokens from $\Y$ as the condition embedding. */}
\State $x_{\mathrm{cond}} \gets \Y[..., -i*i:]$
\State {\color{blue} /* Generate first-order with $\mathsf{FM}_{\mathrm{first}}$. */} \label{line:inference_flow_matching_pred_start}
\State $\wh{y}_{\mathrm{first}} \gets \mathsf{FM}_{\mathrm{first}}(x_{\mathrm{cond}}, x_{\mathrm{img}})$
\State {\color{blue} /* Generate second-order with $\mathsf{FM}_{\mathrm{second}}$. */}
\State $\wh{y}_{second} \gets \mathsf{FM}_{\mathrm{second}}(x_{\mathrm{cond}}, x_{\mathrm{img}})$ \label{line:inference_flow_matching_pred_end}
\State {\color{blue} /* Apply first and second-order terms. */}
\State $x_{\mathrm{img}} \gets x_{\mathrm{img}} + \wh{y}_{\mathrm{first}} \cdot \Delta t + 0.5 \cdot \wh{y}_{\mathrm{second}} \cdot (\Delta t)^2$ \label{line:apply_first_second_term}
\State {\color{blue} /* Upsample $x_{\mathrm{img}}$. */}
\State $x_{\mathrm{img}} \gets \phi_{\mathrm{up}}(x_{\mathrm{img}})$
\State {\color{blue} /* Concatenate upsampled $x_{\mathrm{img}}$ to the Transformer input. */}
\State $x \gets \mathsf{Concat}(x, x_{\mathrm{img}})$ \label{line:inference_loop_end}
\EndFor
\State {\color{blue} /* Return the final image */}
\State \Return{$x_{\mathrm{img}}$}
\EndProcedure
\end{algorithmic}
\end{algorithm}

\section{Technical Overview} \label{sec:technical_overview}

In this section, we present the key lemmas used to prove the main theorem introduced in the previous section. Specifically, we first analyze the computational complexity of each component in auto-regressive Transformers and the Flow-Matching architecture. Then, we integrate these results to derive the overall runtime for both the Transformer and Flow-Matching components.

We begin by analyzing the runtime of the auto-regressive Transformer module.
\begin{lemma} [Running time for Auto-Regressive Transformer Forward] \label{lem:autoregressive_transformer_forward_running_time}
Let the auto-regressive Transformer is defined as in Definition~\ref{def:ar_transformer} and that it contains $m$ attention layers. Let $x_{\mathrm{img}} \in \R^{n \times n \times c}$ be the input image, where $n$ denotes the resolution and $c$ denotes the number of channels, and let $d$ denote the hidden dimension. Under these conditions, the running time for a single forward pass of the auto-regressive Transformer is $O(mn^4d)$.
\end{lemma}

\begin{proof}
We consider each attention block in the Transformers architecture. 

For each attention block, it consists of the following three steps:

{\bf Step 1: Generate matrices $Q, K, V$.} We need to generate a query vector $q \in \R^d$, a key vector $k \in \R^d$ and a value vector $v \in \R^d$ for each pixel in the original $n \times n$ image $x_{\mathrm{img}}$. After this step, we will have three matrices $Q, K, V \in \R^{n^2 \times d}$. This step takes $O(n^2d)$ time. 

{\bf Step 2: Calculate the attention matrix.} As defined in Definition~\ref{def:attn_layer}, we need to calculate the attention matrix. It takes $O(n^4d)$ to calculate $Q K^\top \in \R^{n^2 \times n^2}$. It takes $O(n^4)$ time to calculate $\exp(QK^\top)$. It takes $O(n^2)$ time to calculate $D = \exp(QK^\top) {\bf 1}_{n^2}$. It takes $O(n^2)$ to calculate the $D^{-1}$. It takes $O(n^4)$ to multiply $D^{-1}$ to each row of $\exp(QK^\top)$. The overall running time is $O(n^4d)$. After this step, we will get the attention matrix $A \in \R^{n^2 \times n^2}$. 

{\bf Step 3: Calculate the final output.} The final step is to calculate $A \cdot V$. Since $A \in \R^{n^2 \times n^2}$ and $V \in \R^{n^2 \times d}$. The running time of this step is $O(n^4 d)$. Therefore, according to the above analysis, the running time for a single attention operation is $O(n^4 d)$. Since there are total $m$ attention layers in the auto-regressive Transformer, the overall running time is $O(mn^4d)$.

\end{proof}

Another crucial component of the HOFAR model is the flow-matching architecture. Following a similar approach, we analyze the computational complexity of the flow-matching model as follows:
\begin{lemma} [Running time for Flow-Matching Forward] \label{lem:flow_matching_forward_running_time}
Let the auto-regressive Transformer $\mathsf{TF}$ be defined as in Definition~\ref{def:ar_transformer} and that the flow-matching architecture is defined as in Definition~\ref{def:flow_matching_architecture}. Let $x_{\mathrm{img}} \in \R^{n \times n \times c}$ denote the image, where $n$ denotes the resolution and $c$ denotes the number of channels, and let $d$ denote the hidden dimension. Under these conditions, we can show that the running time for a single forward pass of the flow-matching architecture is $O(n^4d^2)$.
\end{lemma}

\begin{proof}
Since the input of the flow-matching is the output of the auto-regressive Transformer, which is $\mathsf{TF}(x_{\mathrm{img}}) \in \R^{n^2 \times n^2 \times d}$. According to the definition of flow-matching architecture (Definition~\ref{def:flow_matching_architecture}), it consists of three operations: one MLP layer, one attention layer, and one MLP layer. For the first layer, the MLP layer, the running complexity is $O(n^2d^2)$. 
For the second layer, the attention layer, according to the proof of Lemma~\ref{lem:autoregressive_transformer_forward_running_time}, the running time for this layer is $O(n^4d)$. 
For the third layer, the MLP layer, the running complexity is $O(n^2d^2)$. Therefore, the overall running time for the flow-matching is $O(n^4d^2)$.

\end{proof}

With the runtime analysis of both the Transformer and Flow-Matching modules completed, we now proceed to analyze the training procedure of the HOFAR model. In the following proof, we break down the training process step by step and derive the overall computational complexity at the end.
\begin{lemma} [Running time for HOFAR training] \label{lem:hofar_training_running_time}
Suppose that the auto-regressive Transformer is defined as in Definition~\ref{def:ar_transformer} and contains $m$ attention layers. Let the flow-matching architecture be defined as in Definition~\ref{def:flow_matching_architecture}, and assume that the HOFAR training process is described in Algorithm~\ref{alg:hofar_training}. Furthermore, suppose that HOFAR consists of $k$ pyramid frames, let $d$ denote the hidden dimension, and let $x_{\mathrm{img}} \in \R^{n \times n \times c}$ denote the image with resolution $n$ and $c$ channels. Then, the running time of the training procedure of HOFAR is $O(kmn^4d^2)$.
\end{lemma}

\begin{proof}
We first consider the running time for each pyramid frame in the training loop (Line~\ref{line:pyramid_loop_start} to \ref{line:pyramid_loop_end} in Algorithm~\ref{alg:hofar_training}). In each loop, we first consider time complexity for the preparation of essential variables (Line~\ref{line:preparation_start} to \ref{line:preparation_end}). Since the dimension of each variable in this process is $n \times n \times d$, the running complexity for the preparation process is $O(n^2d)$. Then, we consider the process of generating condition embeddings with Transformer (Line~\ref{line:generate_condition_with_transformer}). According to Lemma~\ref{lem:autoregressive_transformer_forward_running_time}, the running time for this process is $O(mn^4d)$. Next, according to Lemma~\ref{lem:flow_matching_forward_running_time}, the prediction process of the flow-matching models takes $O(n^4d^2)$ time. Finally, the loss calculation step (Line~\ref{line:loss_calculation}) takes $O(n^2d)$ time. 

Therefore, according to all the analysis mentioned above, the running time for each iteration is $O(mn^4d^2)$. 

Since there are total $k$ pyramid frames, the overall running time for the training process is $O(kmn^4d^2)$. 

\end{proof}

Following a similar procedure, we can have the running complexity analysis for the inference procedure as follows:
\begin{lemma} [Running time for HOFAR inference] \label{lem:hofar_inference_running_time}
Let the auto-regressive Transformer be defined as in Definition~\ref{def:ar_transformer} and contain $m$ attention layers, that the flow-matching architecture is defined as in Definition~\ref{def:flow_matching_architecture}, and that the HOFAR inference process is described in Algorithm~\ref{alg:hofar_inference}. Also, suppose there are $k$ pyramid frames in HOFAR and let $d$ denote the hidden dimension. Under these conditions, the running time of the HOFAR inference procedure is $O(kmn^4d^2)$.
\end{lemma}

\begin{proof}
We begin with considering each loop in $k$-th inference (Line~\ref{line:inference_loop_start} to Line~\ref{line:inference_loop_end}). For each loop, according to Lemma~\ref{lem:autoregressive_transformer_forward_running_time}, the Transformer forward pass (Line~\ref{line:inference_transformer_pass}) takes $O(mn^4d)$ time. Next, according to Lemma~\ref{lem:flow_matching_forward_running_time}, the flow-matching prediction process (Line~\ref{line:inference_flow_matching_pred_start} - \ref{line:inference_flow_matching_pred_end}) takes $O(n^4d^2)$ time. Finally, the time for applying the predicted gradient on image (Line~\ref{line:apply_first_second_term}) takes $O(n^2d)$ time. Therefore, the overall running time for each inference loop is $O(mn^4d^2)$. Since there are total $k$ inference loops, the overall running time is $O(kmn^4d^2)$.

\end{proof}

Combining all the analyses discussed above, we can directly arrive at our final theorem (Theorem~\ref{thm:hofar_computational_efficiency}).

\section{Experiments} \label{sec:experiments}

\begin{figure}[!ht]
    \centering
    \includegraphics[width=0.45\linewidth]{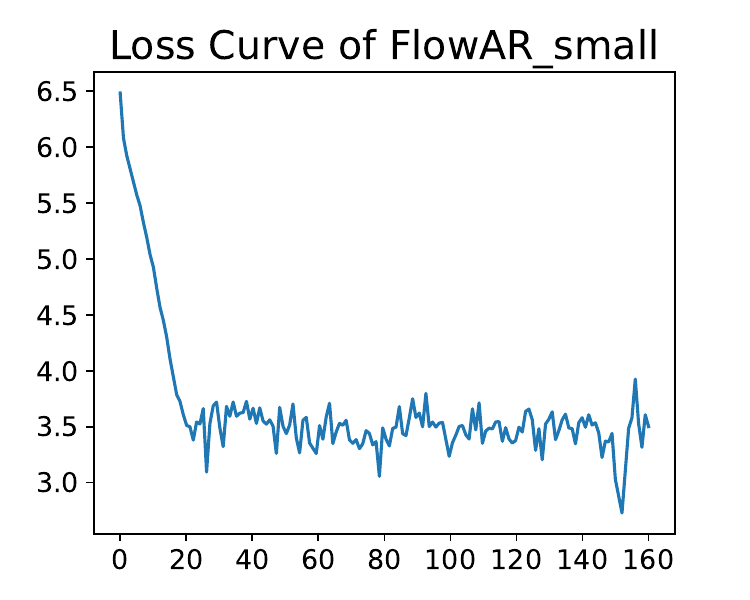}
    \includegraphics[width=0.45\linewidth]{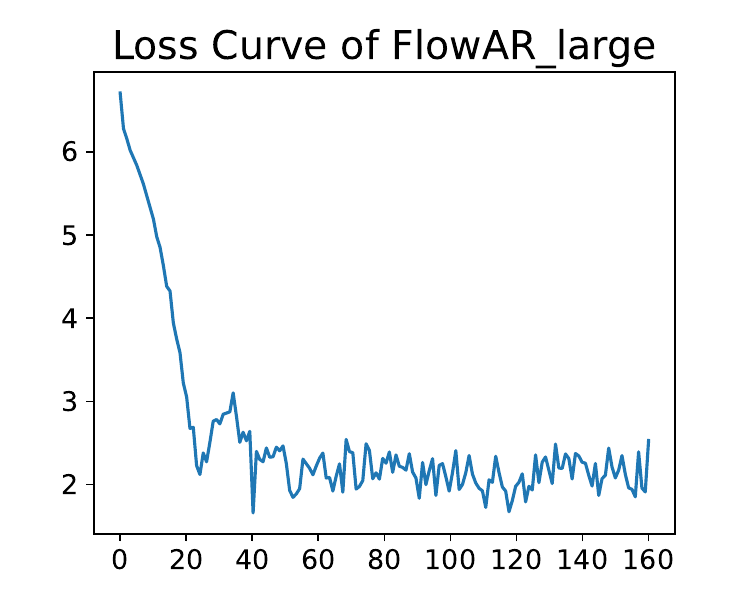}     
    \includegraphics[width=0.45\linewidth]{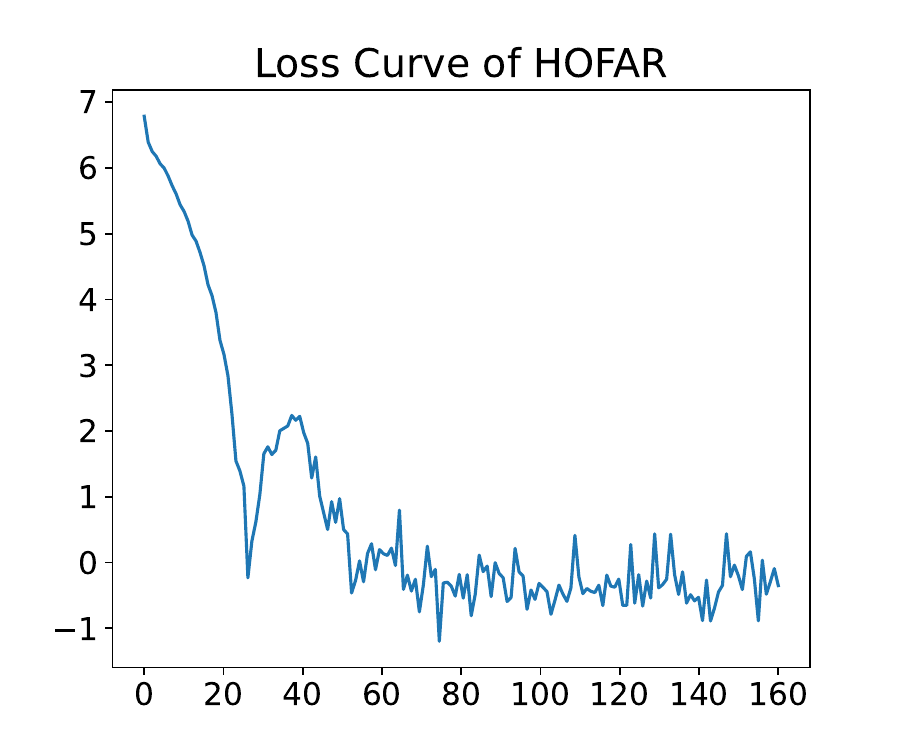}
    \caption{Loss curve of FlowAR-small (\textbf{Left}), loss curve of FlowAR-large (\textbf{Right}) and loss curve of HOFAR (\textbf{Bottom}). }
    \label{fig:app:loss_curve}
\end{figure}

\ifdefined\isarxiv
\newcommand{\mywidth}{0.05}
\else
\newcommand{\mywidth}{0.0581}
\fi

\begin{figure*}[!ht]
\centering

\includegraphics[width=\mywidth\linewidth]{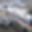}
\includegraphics[width=\mywidth\linewidth]{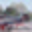}
\includegraphics[width=\mywidth\linewidth]{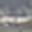}
\includegraphics[width=\mywidth\linewidth]{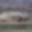}
\includegraphics[width=\mywidth\linewidth]{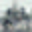}
\includegraphics[width=\mywidth\linewidth]{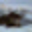}
\includegraphics[width=\mywidth\linewidth]{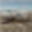}
\includegraphics[width=\mywidth\linewidth]{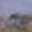}
\includegraphics[width=\mywidth\linewidth]{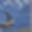}
\includegraphics[width=\mywidth\linewidth]{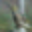}
\includegraphics[width=\mywidth\linewidth]{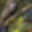}
\includegraphics[width=\mywidth\linewidth]{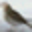}
\includegraphics[width=\mywidth\linewidth]{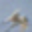}
\includegraphics[width=\mywidth\linewidth]{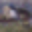}
\includegraphics[width=\mywidth\linewidth]{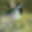}
\includegraphics[width=\mywidth\linewidth]{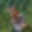} \\
\includegraphics[width=\mywidth\linewidth]{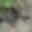}
\includegraphics[width=\mywidth\linewidth]{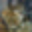}
\includegraphics[width=\mywidth\linewidth]{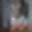}
\includegraphics[width=\mywidth\linewidth]{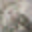}
\includegraphics[width=\mywidth\linewidth]{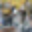}
\includegraphics[width=\mywidth\linewidth]{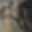}
\includegraphics[width=\mywidth\linewidth]{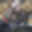}
\includegraphics[width=\mywidth\linewidth]{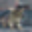}
\includegraphics[width=\mywidth\linewidth]{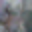}
\includegraphics[width=\mywidth\linewidth]{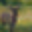}
\includegraphics[width=\mywidth\linewidth]{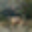}
\includegraphics[width=\mywidth\linewidth]{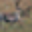}
\includegraphics[width=\mywidth\linewidth]{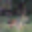}
\includegraphics[width=\mywidth\linewidth]{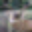}
\includegraphics[width=\mywidth\linewidth]{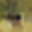}
\includegraphics[width=\mywidth\linewidth]{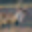} \\
\includegraphics[width=\mywidth\linewidth]{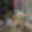}
\includegraphics[width=\mywidth\linewidth]{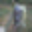}
\includegraphics[width=\mywidth\linewidth]{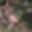}
\includegraphics[width=\mywidth\linewidth]{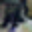}
\includegraphics[width=\mywidth\linewidth]{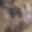}
\includegraphics[width=\mywidth\linewidth]{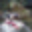}
\includegraphics[width=\mywidth\linewidth]{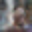}
\includegraphics[width=\mywidth\linewidth]{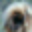}
\includegraphics[width=\mywidth\linewidth]{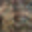}
\includegraphics[width=\mywidth\linewidth]{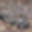}
\includegraphics[width=\mywidth\linewidth]{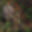}
\includegraphics[width=\mywidth\linewidth]{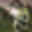}
\includegraphics[width=\mywidth\linewidth]{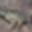}
\includegraphics[width=\mywidth\linewidth]{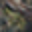}
\includegraphics[width=\mywidth\linewidth]{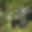}
\includegraphics[width=\mywidth\linewidth]{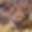} \\
\includegraphics[width=\mywidth\linewidth]{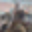}
\includegraphics[width=\mywidth\linewidth]{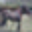}
\includegraphics[width=\mywidth\linewidth]{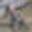}
\includegraphics[width=\mywidth\linewidth]{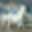}
\includegraphics[width=\mywidth\linewidth]{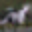}
\includegraphics[width=\mywidth\linewidth]{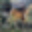}
\includegraphics[width=\mywidth\linewidth]{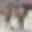}
\includegraphics[width=\mywidth\linewidth]{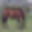} 
\includegraphics[width=\mywidth\linewidth]{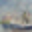}
\includegraphics[width=\mywidth\linewidth]{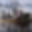}
\includegraphics[width=\mywidth\linewidth]{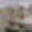}
\includegraphics[width=\mywidth\linewidth]{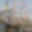}
\includegraphics[width=\mywidth\linewidth]{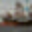}
\includegraphics[width=\mywidth\linewidth]{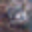}
\includegraphics[width=\mywidth\linewidth]{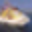}
\includegraphics[width=\mywidth\linewidth]{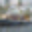} \\

(a) Images generated by FlowAR-small. \\ [1.0mm]

\includegraphics[width=\mywidth\linewidth]{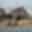}
\includegraphics[width=\mywidth\linewidth]{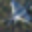}
\includegraphics[width=\mywidth\linewidth]{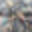}
\includegraphics[width=\mywidth\linewidth]{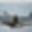}
\includegraphics[width=\mywidth\linewidth]{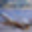}
\includegraphics[width=\mywidth\linewidth]{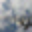}
\includegraphics[width=\mywidth\linewidth]{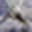}
\includegraphics[width=\mywidth\linewidth]{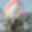}
\includegraphics[width=\mywidth\linewidth]{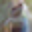}
\includegraphics[width=\mywidth\linewidth]{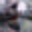}
\includegraphics[width=\mywidth\linewidth]{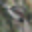}
\includegraphics[width=\mywidth\linewidth]{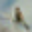}
\includegraphics[width=\mywidth\linewidth]{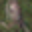}
\includegraphics[width=\mywidth\linewidth]{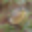}
\includegraphics[width=\mywidth\linewidth]{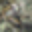}
\includegraphics[width=\mywidth\linewidth]{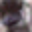} \\
\includegraphics[width=\mywidth\linewidth]{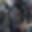}
\includegraphics[width=\mywidth\linewidth]{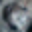}
\includegraphics[width=\mywidth\linewidth]{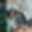}
\includegraphics[width=\mywidth\linewidth]{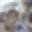}
\includegraphics[width=\mywidth\linewidth]{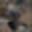}
\includegraphics[width=\mywidth\linewidth]{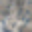}
\includegraphics[width=\mywidth\linewidth]{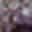}
\includegraphics[width=\mywidth\linewidth]{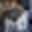}
\includegraphics[width=\mywidth\linewidth]{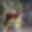}
\includegraphics[width=\mywidth\linewidth]{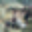}
\includegraphics[width=\mywidth\linewidth]{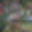}
\includegraphics[width=\mywidth\linewidth]{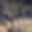}
\includegraphics[width=\mywidth\linewidth]{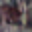}
\includegraphics[width=\mywidth\linewidth]{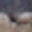}
\includegraphics[width=\mywidth\linewidth]{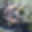}
\includegraphics[width=\mywidth\linewidth]{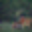} \\
\includegraphics[width=\mywidth\linewidth]{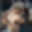}
\includegraphics[width=\mywidth\linewidth]{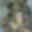}
\includegraphics[width=\mywidth\linewidth]{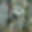}
\includegraphics[width=\mywidth\linewidth]{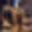}
\includegraphics[width=\mywidth\linewidth]{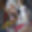}
\includegraphics[width=\mywidth\linewidth]{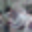}
\includegraphics[width=\mywidth\linewidth]{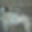}
\includegraphics[width=\mywidth\linewidth]{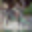}
\includegraphics[width=\mywidth\linewidth]{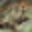}
\includegraphics[width=\mywidth\linewidth]{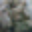}
\includegraphics[width=\mywidth\linewidth]{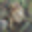}
\includegraphics[width=\mywidth\linewidth]{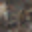}
\includegraphics[width=\mywidth\linewidth]{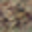}
\includegraphics[width=\mywidth\linewidth]{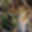}
\includegraphics[width=\mywidth\linewidth]{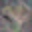}
\includegraphics[width=\mywidth\linewidth]{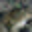} \\
\includegraphics[width=\mywidth\linewidth]{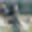}
\includegraphics[width=\mywidth\linewidth]{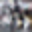}
\includegraphics[width=\mywidth\linewidth]{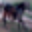}
\includegraphics[width=\mywidth\linewidth]{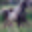}
\includegraphics[width=\mywidth\linewidth]{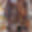}
\includegraphics[width=\mywidth\linewidth]{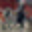}
\includegraphics[width=\mywidth\linewidth]{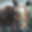}
\includegraphics[width=\mywidth\linewidth]{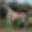} 
\includegraphics[width=\mywidth\linewidth]{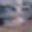}
\includegraphics[width=\mywidth\linewidth]{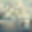}
\includegraphics[width=\mywidth\linewidth]{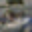}
\includegraphics[width=\mywidth\linewidth]{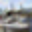}
\includegraphics[width=\mywidth\linewidth]{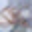}
\includegraphics[width=\mywidth\linewidth]{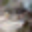}
\includegraphics[width=\mywidth\linewidth]{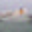}
\includegraphics[width=\mywidth\linewidth]{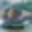} \\

(b) Images generated by FlowAR-large. \\ [1.0mm]

\includegraphics[width=\mywidth\linewidth]{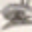}
\includegraphics[width=\mywidth\linewidth]{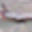}
\includegraphics[width=\mywidth\linewidth]{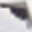}
\includegraphics[width=\mywidth\linewidth]{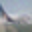}
\includegraphics[width=\mywidth\linewidth]{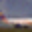}
\includegraphics[width=\mywidth\linewidth]{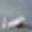}
\includegraphics[width=\mywidth\linewidth]{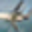}
\includegraphics[width=\mywidth\linewidth]{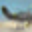} 
\includegraphics[width=\mywidth\linewidth]{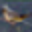}
\includegraphics[width=\mywidth\linewidth]{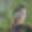}
\includegraphics[width=\mywidth\linewidth]{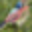}
\includegraphics[width=\mywidth\linewidth]{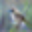}
\includegraphics[width=\mywidth\linewidth]{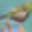}
\includegraphics[width=\mywidth\linewidth]{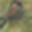}
\includegraphics[width=\mywidth\linewidth]{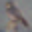}
\includegraphics[width=\mywidth\linewidth]{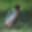} \\
\includegraphics[width=\mywidth\linewidth]{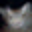}
\includegraphics[width=\mywidth\linewidth]{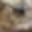}
\includegraphics[width=\mywidth\linewidth]{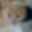}
\includegraphics[width=\mywidth\linewidth]{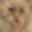}
\includegraphics[width=\mywidth\linewidth]{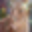}
\includegraphics[width=\mywidth\linewidth]{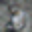}
\includegraphics[width=\mywidth\linewidth]{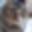}
\includegraphics[width=\mywidth\linewidth]{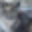}
\includegraphics[width=\mywidth\linewidth]{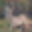}
\includegraphics[width=\mywidth\linewidth]{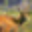}
\includegraphics[width=\mywidth\linewidth]{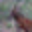}
\includegraphics[width=\mywidth\linewidth]{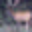}
\includegraphics[width=\mywidth\linewidth]{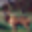}
\includegraphics[width=\mywidth\linewidth]{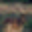}
\includegraphics[width=\mywidth\linewidth]{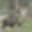}
\includegraphics[width=\mywidth\linewidth]{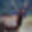} \\
\includegraphics[width=\mywidth\linewidth]{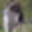}
\includegraphics[width=\mywidth\linewidth]{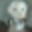}
\includegraphics[width=\mywidth\linewidth]{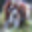}
\includegraphics[width=\mywidth\linewidth]{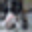}
\includegraphics[width=\mywidth\linewidth]{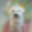}
\includegraphics[width=\mywidth\linewidth]{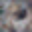}
\includegraphics[width=\mywidth\linewidth]{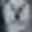}
\includegraphics[width=\mywidth\linewidth]{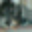}
\includegraphics[width=\mywidth\linewidth]{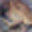}
\includegraphics[width=\mywidth\linewidth]{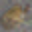}
\includegraphics[width=\mywidth\linewidth]{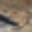}
\includegraphics[width=\mywidth\linewidth]{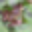}
\includegraphics[width=\mywidth\linewidth]{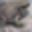}
\includegraphics[width=\mywidth\linewidth]{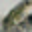}
\includegraphics[width=\mywidth\linewidth]{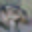}
\includegraphics[width=\mywidth\linewidth]{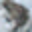} \\
\includegraphics[width=\mywidth\linewidth]{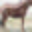}
\includegraphics[width=\mywidth\linewidth]{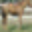}
\includegraphics[width=\mywidth\linewidth]{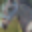}
\includegraphics[width=\mywidth\linewidth]{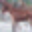}
\includegraphics[width=\mywidth\linewidth]{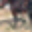}
\includegraphics[width=\mywidth\linewidth]{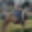}
\includegraphics[width=\mywidth\linewidth]{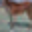}
\includegraphics[width=\mywidth\linewidth]{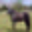} 
\includegraphics[width=\mywidth\linewidth]{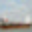}
\includegraphics[width=\mywidth\linewidth]{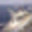}
\includegraphics[width=\mywidth\linewidth]{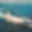}
\includegraphics[width=\mywidth\linewidth]{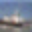}
\includegraphics[width=\mywidth\linewidth]{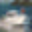}
\includegraphics[width=\mywidth\linewidth]{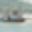}
\includegraphics[width=\mywidth\linewidth]{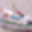}
\includegraphics[width=\mywidth\linewidth]{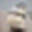} \\

(c) Images generated by HOFAR. \\ 

\caption{Comparison of 32*32 CIFAR-10 images generation by FlowAR-small (\textbf{first four lines}), FlowAR-large (\textbf{second four lines}) and HOFAR (\textbf{last four lines}). For better looking, we put higher-resolution version of Figure~\ref{fig:app:visual_flowar}, Figure~\ref{fig:app:visual_flowar_l} and Figure~\ref{fig:app:visual_hofar} here. }
\label{fig:comparision}
\end{figure*}

In Section ~\ref{sec:app:experiment_setup}, we introduce the setting we used in our experiments. In Section~\ref{sec:app:loss_curve}, we present the loss curve of the various models. In Section~\ref{sec:visualization_comparison}, we present visualization examples produced by the FlowAR-small, FlowAR-large and HOFAR, highlighting differences in color accuracy and generation quality on CIFAR-10 images. 

\subsection{Experiment Setup}\label{sec:app:experiment_setup}
In FlowAR-small, we employ an embedding with three dimensions, and its Autoregressive component is configured with a 1024-dimensional feature space across a depth of 2 layers. Additionally, the flow-matching component is realized through a single hidden layer MLP operating with a step increment of 25. 
By comparison, FlowAR-large distinguishes itself by utilizing an eight dimension embedding and extending the Autoregressive feature dimension to 1536, while retaining the same configuration for the remaining components as in FlowAR-small. 
In the case of HOFAR, an embedding of dimension three is similarly adopted, paired with a 1024 dimension Autoregressive component structured over two layers, and a single-hidden-layer MLP is again employed for flow-matching with 25 steps. All three models were evaluated on the CIFAR-10 dataset, with analysis restricted to 8 classes due to computational constraints. 
All models above use AdamW optimizer with 0.0001 learning rate. 
In all experiments, the models were optimized by minimizing the sum of squared errors (SSE), and performance assessment during testing was based on the Euclidean distance metric. 
Regarding the target transport trajectory, we integrated the VP ODE framework as described in~\cite{lgl22}, represented by 
$x_t = \alpha_t x_0 + \beta_t x_1$. Here, $\alpha_t$ is defined as 
$\exp\left(-\frac{1}{4}a(1-t)^2 - \frac{1}{2}b(1-t)\right)$, and $\beta_t$ is determined by 
$\sqrt{1 - \alpha_t^2}$,with the hyperparameters fixed at $a = 19.9$ and $b = 0.1$. 
During generation, the eight distinct training labels were provided as input, and a consistent $cfg$ value of 4.3 was maintained for all three models.

\subsection{Loss Function Curve}\label{sec:app:loss_curve}
Now, we present the testing loss curves of the various models during training, providing insights into their convergence behavior and learning dynamics. Figure \ref{fig:app:loss_curve} illustrates the loss for FlowAR-small, FlowAR-large, and our HOFAR, with the respective model parameter counts being 170.70M, 222.72M, and 212.44M.

\subsection{Visualization Comparison}\label{sec:visualization_comparison}
As Figure~\ref {fig:comparision} shows, the visualization instances generated by the FlowAR-small, FlowAR-large and HOFAR models are delineated in this study. Each model uses the same prompt at the corresponding position.

\section{Conclusion}\label{sec:conclusion}

In this work, we presented High-Order FlowAR (HOFAR), a novel framework that integrates high-order dynamics into flow-matching-based auto-regressive generation. By modeling higher-order interactions, HOFAR enhances the ability to capture complex dependencies, leading to improved realism, coherence in generative tasks. Our theoretical analysis demonstrates that HOFAR maintains computational efficiency while benefiting from high-order. Empirical evaluations further validate the superiority of HOFAR over existing auto-regressive generative models. These contributions highlight the potential of incorporating high-order dynamics into generative frameworks, paving the way for more advanced generative models in the future.

\ifdefined\isarxiv
%\section*{Acknowledgments}
\bibliographystyle{alpha}
\bibliography{ref}
\else

{
    \small
    \bibliographystyle{ieeenat_fullname}
    \bibliography{ref}
}

\fi

\newpage
\onecolumn
\appendix

\begin{center}
    \textbf{\LARGE Appendix }
\end{center}
\paragraph{Roadmap.}  
In Section~\ref{sec:app:discussion}, we analyze the strengths and limitations of the High-Order FlowAR (HOFAR) framework. 
In Section~\ref{sec:app:empirical_result}, we exhibit some result obtained from the experiments.

\section{Discussion} \label{sec:app:discussion}

The HOFAR framework introduces a novel approach to integrating high-order dynamics into flow-matching-based auto-regressive generation, significantly improving the modeling of complex dependencies and generation quality. However, certain limitations and future directions deserve attention. 
One limitation is the potential computational overhead when scaling HOFAR to extremely high-dimensional data, such as ultra-high-resolution images or long-duration videos. While HOFAR maintains theoretical efficiency, practical implementation may require further optimization to handle such scenarios.
Future work could explore extending HOFAR to multi-modal generation tasks, such as joint text-video or text-3D generation, where capturing long-term coherence across modalities is critical. 
Furthermore, improving the interpretability of high-order dynamics through visualization or disentanglement techniques would broaden HOFAR's applicability.

\section{Empirical Result}\label{sec:app:empirical_result}
In Section~\ref{sec:app:visualization_examples}, we compare visualizations generated by FlowAR and our HOFAR, this highlighting differences in color accuracy and relative position on CIFAR-10 images.

\subsection{Visualization Examples}\label{sec:app:visualization_examples}
We present visualization examples produced by the FlowAR-small, FlowAR-large and proposed HOFAR. Specifically, Figure~\ref{fig:app:visual_flowar} showcases visualizations generated by the FlowAR-small model, Figure~\ref{fig:app:visual_flowar_l} showcases visualizations generated by the FlowAR-large model, whereas Figure~\ref{fig:app:visual_hofar} highlights visualizations created by the HOFAR model.

\begin{figure}[!ht]
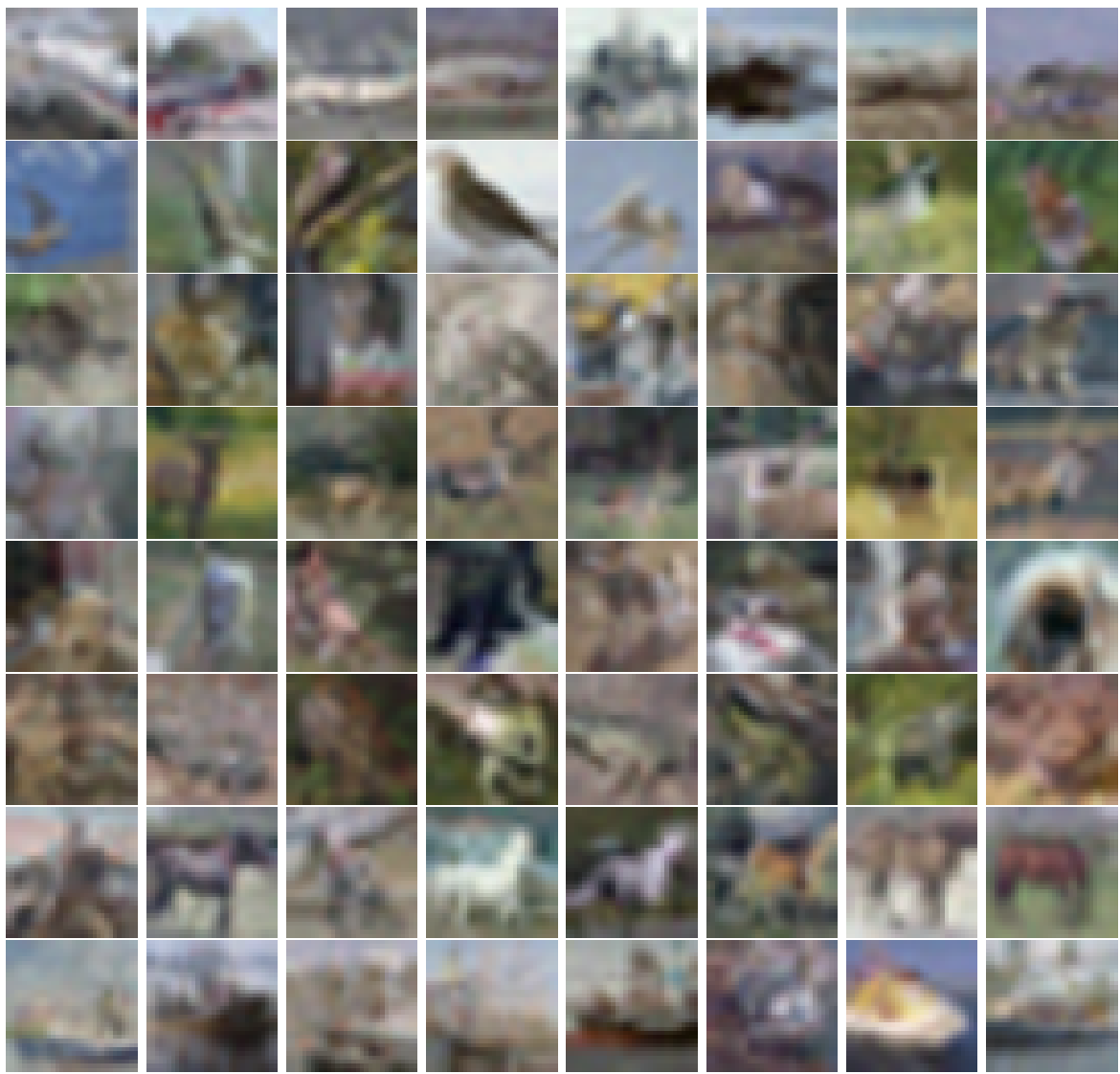

\centering
\includegraphics[width=0.11\linewidth]{figs/00_flowar.png}
\includegraphics[width=0.11\linewidth]{figs/01_flowar.png}
\includegraphics[width=0.11\linewidth]{figs/02_flowar.png}
\includegraphics[width=0.11\linewidth]{figs/03_flowar.png}
\includegraphics[width=0.11\linewidth]{figs/04_flowar.png}
\includegraphics[width=0.11\linewidth]{figs/05_flowar.png}
\includegraphics[width=0.11\linewidth]{figs/06_flowar.png}
\includegraphics[width=0.11\linewidth]{figs/07_flowar.png} \\
\includegraphics[width=0.11\linewidth]{figs/10_flowar.png}
\includegraphics[width=0.11\linewidth]{figs/11_flowar.png}
\includegraphics[width=0.11\linewidth]{figs/12_flowar.png}
\includegraphics[width=0.11\linewidth]{figs/13_flowar.png}
\includegraphics[width=0.11\linewidth]{figs/14_flowar.png}
\includegraphics[width=0.11\linewidth]{figs/15_flowar.png}
\includegraphics[width=0.11\linewidth]{figs/16_flowar.png}
\includegraphics[width=0.11\linewidth]{figs/17_flowar.png} \\
\includegraphics[width=0.11\linewidth]{figs/20_flowar.png}
\includegraphics[width=0.11\linewidth]{figs/21_flowar.png}
\includegraphics[width=0.11\linewidth]{figs/22_flowar.png}
\includegraphics[width=0.11\linewidth]{figs/23_flowar.png}
\includegraphics[width=0.11\linewidth]{figs/24_flowar.png}
\includegraphics[width=0.11\linewidth]{figs/25_flowar.png}
\includegraphics[width=0.11\linewidth]{figs/26_flowar.png}
\includegraphics[width=0.11\linewidth]{figs/27_flowar.png} \\
\includegraphics[width=0.11\linewidth]{figs/30_flowar.png}
\includegraphics[width=0.11\linewidth]{figs/31_flowar.png}
\includegraphics[width=0.11\linewidth]{figs/32_flowar.png}
\includegraphics[width=0.11\linewidth]{figs/33_flowar.png}
\includegraphics[width=0.11\linewidth]{figs/34_flowar.png}
\includegraphics[width=0.11\linewidth]{figs/35_flowar.png}
\includegraphics[width=0.11\linewidth]{figs/36_flowar.png}
\includegraphics[width=0.11\linewidth]{figs/37_flowar.png} \\
\includegraphics[width=0.11\linewidth]{figs/40_flowar.png}
\includegraphics[width=0.11\linewidth]{figs/41_flowar.png}
\includegraphics[width=0.11\linewidth]{figs/42_flowar.png}
\includegraphics[width=0.11\linewidth]{figs/43_flowar.png}
\includegraphics[width=0.11\linewidth]{figs/44_flowar.png}
\includegraphics[width=0.11\linewidth]{figs/45_flowar.png}
\includegraphics[width=0.11\linewidth]{figs/46_flowar.png}
\includegraphics[width=0.11\linewidth]{figs/47_flowar.png} \\
\includegraphics[width=0.11\linewidth]{figs/50_flowar.png}
\includegraphics[width=0.11\linewidth]{figs/51_flowar.png}
\includegraphics[width=0.11\linewidth]{figs/52_flowar.png}
\includegraphics[width=0.11\linewidth]{figs/53_flowar.png}
\includegraphics[width=0.11\linewidth]{figs/54_flowar.png}
\includegraphics[width=0.11\linewidth]{figs/55_flowar.png}
\includegraphics[width=0.11\linewidth]{figs/56_flowar.png}
\includegraphics[width=0.11\linewidth]{figs/57_flowar.png} \\
\includegraphics[width=0.11\linewidth]{figs/60_flowar.png}
\includegraphics[width=0.11\linewidth]{figs/61_flowar.png}
\includegraphics[width=0.11\linewidth]{figs/62_flowar.png}
\includegraphics[width=0.11\linewidth]{figs/63_flowar.png}
\includegraphics[width=0.11\linewidth]{figs/64_flowar.png}
\includegraphics[width=0.11\linewidth]{figs/65_flowar.png}
\includegraphics[width=0.11\linewidth]{figs/66_flowar.png}
\includegraphics[width=0.11\linewidth]{figs/67_flowar.png} \\
\includegraphics[width=0.11\linewidth]{figs/70_flowar.png}
\includegraphics[width=0.11\linewidth]{figs/71_flowar.png}
\includegraphics[width=0.11\linewidth]{figs/72_flowar.png}
\includegraphics[width=0.11\linewidth]{figs/73_flowar.png}
\includegraphics[width=0.11\linewidth]{figs/74_flowar.png}
\includegraphics[width=0.11\linewidth]{figs/75_flowar.png}
\includegraphics[width=0.11\linewidth]{figs/76_flowar.png}
\includegraphics[width=0.11\linewidth]{figs/77_flowar.png}
\caption{64 32*32 images generated by FlowAR-small. }
\label{fig:app:visual_flowar}
\end{figure}

\begin{figure}[!ht]
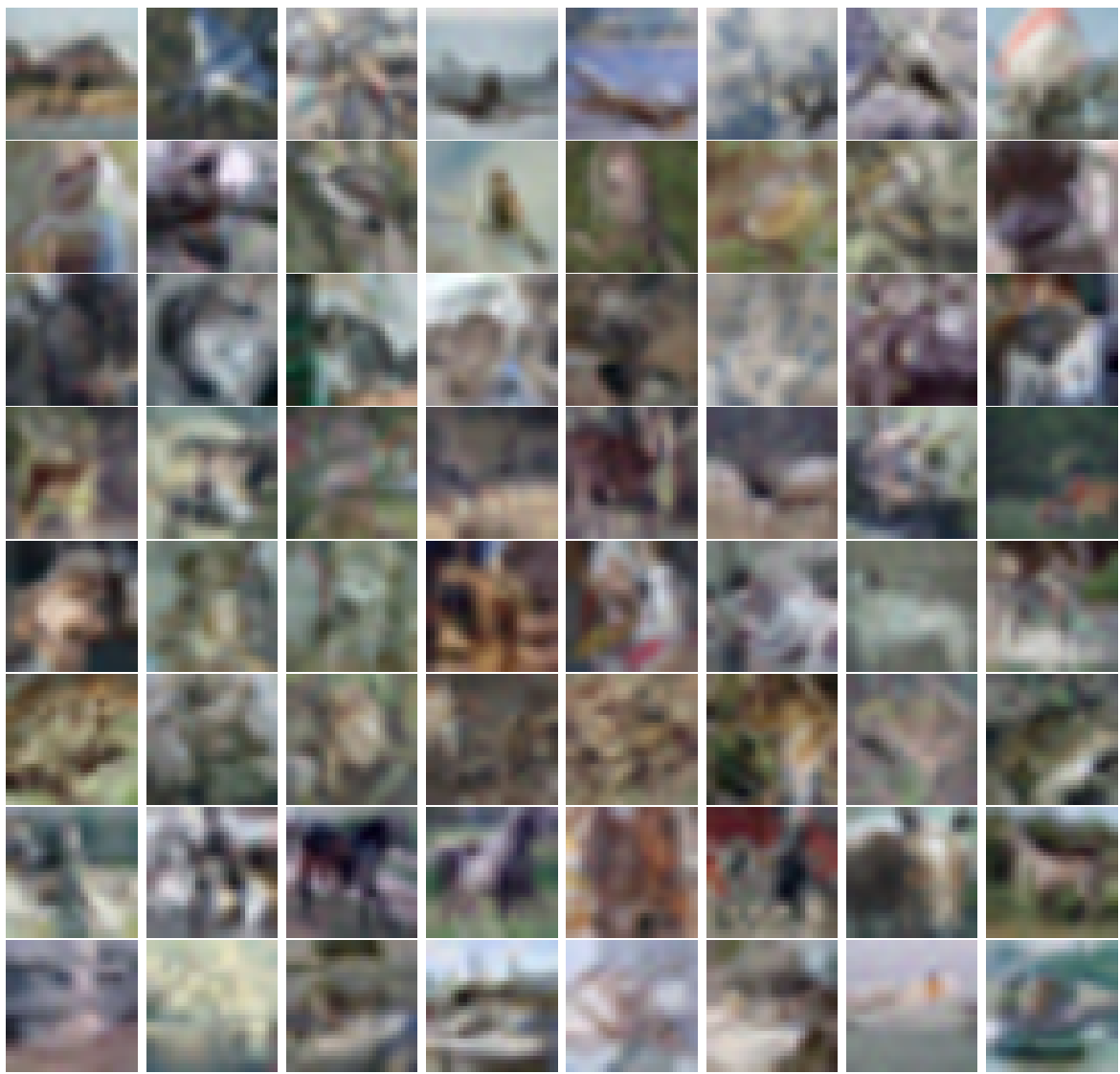

\centering
\includegraphics[width=0.11\linewidth]{figs/00_flowar_l.png}
\includegraphics[width=0.11\linewidth]{figs/01_flowar_l.png}
\includegraphics[width=0.11\linewidth]{figs/02_flowar_l.png}
\includegraphics[width=0.11\linewidth]{figs/03_flowar_l.png}
\includegraphics[width=0.11\linewidth]{figs/04_flowar_l.png}
\includegraphics[width=0.11\linewidth]{figs/05_flowar_l.png}
\includegraphics[width=0.11\linewidth]{figs/06_flowar_l.png}
\includegraphics[width=0.11\linewidth]{figs/07_flowar_l.png} \\
\includegraphics[width=0.11\linewidth]{figs/10_flowar_l.png}
\includegraphics[width=0.11\linewidth]{figs/11_flowar_l.png}
\includegraphics[width=0.11\linewidth]{figs/12_flowar_l.png}
\includegraphics[width=0.11\linewidth]{figs/13_flowar_l.png}
\includegraphics[width=0.11\linewidth]{figs/14_flowar_l.png}
\includegraphics[width=0.11\linewidth]{figs/15_flowar_l.png}
\includegraphics[width=0.11\linewidth]{figs/16_flowar_l.png}
\includegraphics[width=0.11\linewidth]{figs/17_flowar_l.png} \\
\includegraphics[width=0.11\linewidth]{figs/20_flowar_l.png}
\includegraphics[width=0.11\linewidth]{figs/21_flowar_l.png}
\includegraphics[width=0.11\linewidth]{figs/22_flowar_l.png}
\includegraphics[width=0.11\linewidth]{figs/23_flowar_l.png}
\includegraphics[width=0.11\linewidth]{figs/24_flowar_l.png}
\includegraphics[width=0.11\linewidth]{figs/25_flowar_l.png}
\includegraphics[width=0.11\linewidth]{figs/26_flowar_l.png}
\includegraphics[width=0.11\linewidth]{figs/27_flowar_l.png} \\
\includegraphics[width=0.11\linewidth]{figs/30_flowar_l.png}
\includegraphics[width=0.11\linewidth]{figs/31_flowar_l.png}
\includegraphics[width=0.11\linewidth]{figs/32_flowar_l.png}
\includegraphics[width=0.11\linewidth]{figs/33_flowar_l.png}
\includegraphics[width=0.11\linewidth]{figs/34_flowar_l.png}
\includegraphics[width=0.11\linewidth]{figs/35_flowar_l.png}
\includegraphics[width=0.11\linewidth]{figs/36_flowar_l.png}
\includegraphics[width=0.11\linewidth]{figs/37_flowar_l.png} \\
\includegraphics[width=0.11\linewidth]{figs/40_flowar_l.png}
\includegraphics[width=0.11\linewidth]{figs/41_flowar_l.png}
\includegraphics[width=0.11\linewidth]{figs/42_flowar_l.png}
\includegraphics[width=0.11\linewidth]{figs/43_flowar_l.png}
\includegraphics[width=0.11\linewidth]{figs/44_flowar_l.png}
\includegraphics[width=0.11\linewidth]{figs/45_flowar_l.png}
\includegraphics[width=0.11\linewidth]{figs/46_flowar_l.png}
\includegraphics[width=0.11\linewidth]{figs/47_flowar_l.png} \\
\includegraphics[width=0.11\linewidth]{figs/50_flowar_l.png}
\includegraphics[width=0.11\linewidth]{figs/51_flowar_l.png}
\includegraphics[width=0.11\linewidth]{figs/52_flowar_l.png}
\includegraphics[width=0.11\linewidth]{figs/53_flowar_l.png}
\includegraphics[width=0.11\linewidth]{figs/54_flowar_l.png}
\includegraphics[width=0.11\linewidth]{figs/55_flowar_l.png}
\includegraphics[width=0.11\linewidth]{figs/56_flowar_l.png}
\includegraphics[width=0.11\linewidth]{figs/57_flowar_l.png} \\
\includegraphics[width=0.11\linewidth]{figs/60_flowar_l.png}
\includegraphics[width=0.11\linewidth]{figs/61_flowar_l.png}
\includegraphics[width=0.11\linewidth]{figs/62_flowar_l.png}
\includegraphics[width=0.11\linewidth]{figs/63_flowar_l.png}
\includegraphics[width=0.11\linewidth]{figs/64_flowar_l.png}
\includegraphics[width=0.11\linewidth]{figs/65_flowar_l.png}
\includegraphics[width=0.11\linewidth]{figs/66_flowar_l.png}
\includegraphics[width=0.11\linewidth]{figs/67_flowar_l.png} \\
\includegraphics[width=0.11\linewidth]{figs/70_flowar_l.png}
\includegraphics[width=0.11\linewidth]{figs/71_flowar_l.png}
\includegraphics[width=0.11\linewidth]{figs/72_flowar_l.png}
\includegraphics[width=0.11\linewidth]{figs/73_flowar_l.png}
\includegraphics[width=0.11\linewidth]{figs/74_flowar_l.png}
\includegraphics[width=0.11\linewidth]{figs/75_flowar_l.png}
\includegraphics[width=0.11\linewidth]{figs/76_flowar_l.png}
\includegraphics[width=0.11\linewidth]{figs/77_flowar_l.png}
\caption{64 32*32 images generated by FlowAR-large. }
\label{fig:app:visual_flowar_l}
\end{figure}

\begin{figure}[!ht]
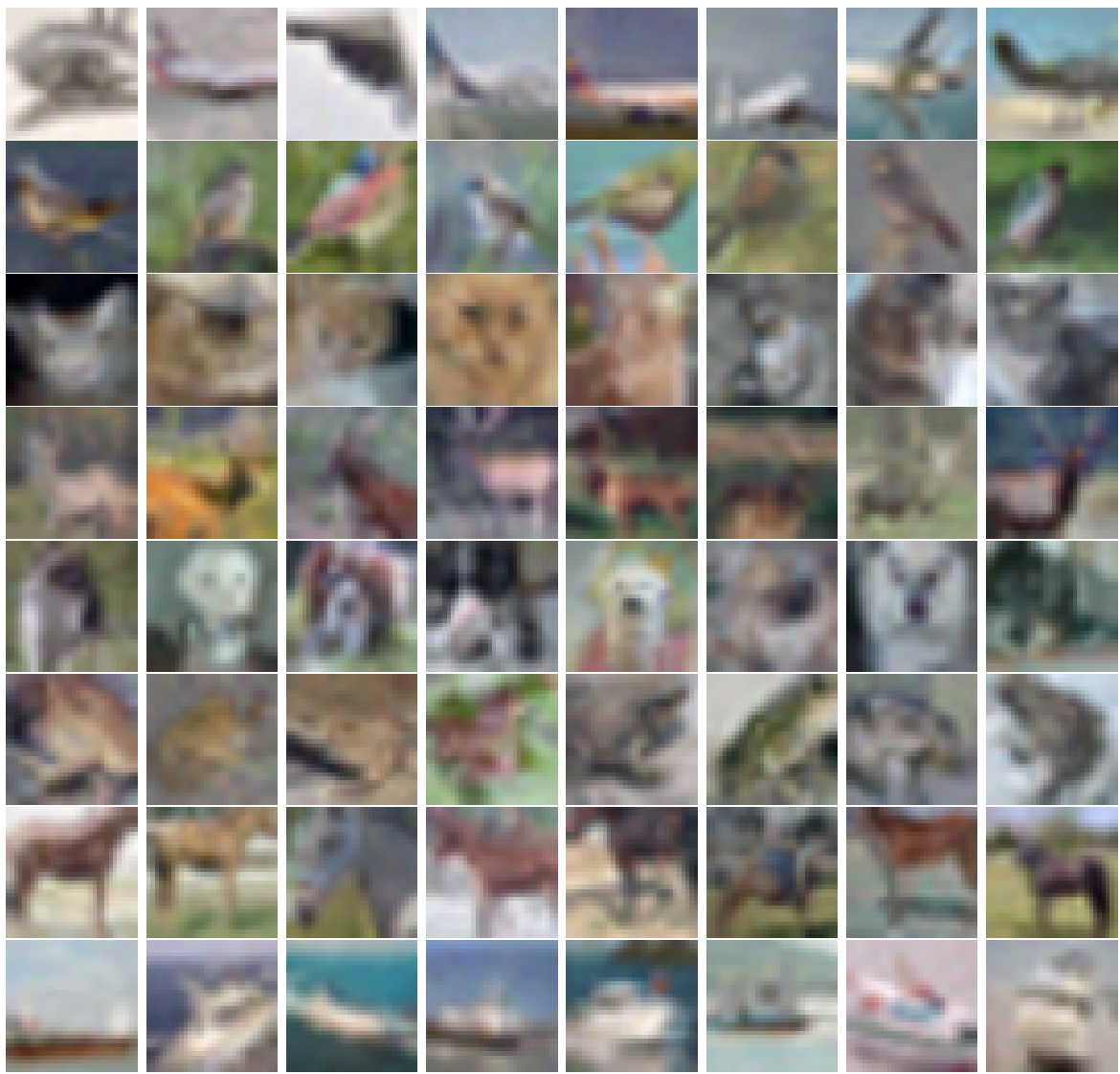

\centering
\includegraphics[width=0.11\linewidth]{figs/00_hofar.png}
\includegraphics[width=0.11\linewidth]{figs/01_hofar.png}
\includegraphics[width=0.11\linewidth]{figs/02_hofar.png}
\includegraphics[width=0.11\linewidth]{figs/03_hofar.png}
\includegraphics[width=0.11\linewidth]{figs/04_hofar.png}
\includegraphics[width=0.11\linewidth]{figs/05_hofar.png}
\includegraphics[width=0.11\linewidth]{figs/06_hofar.png}
\includegraphics[width=0.11\linewidth]{figs/07_hofar.png} \\
\includegraphics[width=0.11\linewidth]{figs/10_hofar.png}
\includegraphics[width=0.11\linewidth]{figs/11_hofar.png}
\includegraphics[width=0.11\linewidth]{figs/12_hofar.png}
\includegraphics[width=0.11\linewidth]{figs/13_hofar.png}
\includegraphics[width=0.11\linewidth]{figs/14_hofar.png}
\includegraphics[width=0.11\linewidth]{figs/15_hofar.png}
\includegraphics[width=0.11\linewidth]{figs/16_hofar.png}
\includegraphics[width=0.11\linewidth]{figs/17_hofar.png} \\
\includegraphics[width=0.11\linewidth]{figs/20_hofar.png}
\includegraphics[width=0.11\linewidth]{figs/21_hofar.png}
\includegraphics[width=0.11\linewidth]{figs/22_hofar.png}
\includegraphics[width=0.11\linewidth]{figs/23_hofar.png}
\includegraphics[width=0.11\linewidth]{figs/24_hofar.png}
\includegraphics[width=0.11\linewidth]{figs/25_hofar.png}
\includegraphics[width=0.11\linewidth]{figs/26_hofar.png}
\includegraphics[width=0.11\linewidth]{figs/27_hofar.png} \\
\includegraphics[width=0.11\linewidth]{figs/30_hofar.png}
\includegraphics[width=0.11\linewidth]{figs/31_hofar.png}
\includegraphics[width=0.11\linewidth]{figs/32_hofar.png}
\includegraphics[width=0.11\linewidth]{figs/33_hofar.png}
\includegraphics[width=0.11\linewidth]{figs/34_hofar.png}
\includegraphics[width=0.11\linewidth]{figs/35_hofar.png}
\includegraphics[width=0.11\linewidth]{figs/36_hofar.png}
\includegraphics[width=0.11\linewidth]{figs/37_hofar.png} \\
\includegraphics[width=0.11\linewidth]{figs/40_hofar.png}
\includegraphics[width=0.11\linewidth]{figs/41_hofar.png}
\includegraphics[width=0.11\linewidth]{figs/42_hofar.png}
\includegraphics[width=0.11\linewidth]{figs/43_hofar.png}
\includegraphics[width=0.11\linewidth]{figs/44_hofar.png}
\includegraphics[width=0.11\linewidth]{figs/45_hofar.png}
\includegraphics[width=0.11\linewidth]{figs/46_hofar.png}
\includegraphics[width=0.11\linewidth]{figs/47_hofar.png} \\
\includegraphics[width=0.11\linewidth]{figs/50_hofar.png}
\includegraphics[width=0.11\linewidth]{figs/51_hofar.png}
\includegraphics[width=0.11\linewidth]{figs/52_hofar.png}
\includegraphics[width=0.11\linewidth]{figs/53_hofar.png}
\includegraphics[width=0.11\linewidth]{figs/54_hofar.png}
\includegraphics[width=0.11\linewidth]{figs/55_hofar.png}
\includegraphics[width=0.11\linewidth]{figs/56_hofar.png}
\includegraphics[width=0.11\linewidth]{figs/57_hofar.png} \\
\includegraphics[width=0.11\linewidth]{figs/60_hofar.png}
\includegraphics[width=0.11\linewidth]{figs/61_hofar.png}
\includegraphics[width=0.11\linewidth]{figs/62_hofar.png}
\includegraphics[width=0.11\linewidth]{figs/63_hofar.png}
\includegraphics[width=0.11\linewidth]{figs/64_hofar.png}
\includegraphics[width=0.11\linewidth]{figs/65_hofar.png}
\includegraphics[width=0.11\linewidth]{figs/66_hofar.png}
\includegraphics[width=0.11\linewidth]{figs/67_hofar.png} \\
\includegraphics[width=0.11\linewidth]{figs/70_hofar.png}
\includegraphics[width=0.11\linewidth]{figs/71_hofar.png}
\includegraphics[width=0.11\linewidth]{figs/72_hofar.png}
\includegraphics[width=0.11\linewidth]{figs/73_hofar.png}
\includegraphics[width=0.11\linewidth]{figs/74_hofar.png}
\includegraphics[width=0.11\linewidth]{figs/75_hofar.png}
\includegraphics[width=0.11\linewidth]{figs/76_hofar.png}
\includegraphics[width=0.11\linewidth]{figs/77_hofar.png}
\caption{64 32*32 images generated by HOFAR. }
\label{fig:app:visual_hofar}
\end{figure}

%%%% Cut-line between first 10 pages and appendix

%%% some writing rules

%% Writing rule for creating tags.
%% Tags :
%% Theorem    \ref{thm:bla_bla}
%% Lemma      \ref{lem:bla_bla}
%% Claim      \ref{cla:bla_bla}
%% Corollary  \ref{cor:bla_bla}
%% Fact       \ref{fac:bla_bla}
%% Definition \ref{def:bla_bla}
%% Section    \ref{sec:bla_bla}
%% Subsection \ref{sub:bla_bla}
%% Equation   \ref{eq:bla_bla}

\end{document}